\documentclass{article} 
\usepackage[final]{colm2026_conference}

\usepackage{microtype}
\usepackage{xcolor}
\usepackage{hyperref}
\usepackage{url}
\usepackage{booktabs}

\usepackage{amsmath}
\usepackage{amssymb}
\usepackage{mathtools}
\usepackage{amsthm}
\usepackage{pifont}

\usepackage[capitalize,noabbrev]{cleveref}

\usepackage{algorithm}
\usepackage{algorithmic}

\theoremstyle{plain}
\newtheorem{theorem}{Theorem}[section]

\newtheorem{lemma}[theorem]{Lemma}

\theoremstyle{definition}

\theoremstyle{remark}
\newtheorem{remark}[theorem]{Remark}

\newcommand{\name}{Online Reasoning Calibration}
\newcommand{\abbr}{ORCA}

\usepackage{lineno}

\definecolor{darkblue}{rgb}{0, 0, 0.5}
\hypersetup{colorlinks=true, citecolor=darkblue, linkcolor=darkblue, urlcolor=darkblue}

\title{\name: Test-Time Training Enables Generalizable Conformal LLM Reasoning}

\author{Cai Zhou\thanks{Equal Contribution.} $^{\ 123}$ \ Zekai Wang$^{*13}$ \ Menghua Wu$^{12}$ \ Qianyu Julie Zhu$^{34}$ \ Flora C. Shi$^{13}$\\\textbf{Chenyu Wang}$^{12}$ \ \textbf{Ashia Wilson}$^{13}$ \ \textbf{Tommi Jaakkola}$^{12}$ \ \textbf{Stephen Bates}$^{13}$\\
$^1$Department of Electrical Engineering and Computer Science (MIT EECS)\\$^2$Computer Science and Artificial Intelligence Laboratory (MIT CSAIL) \\$^3$Laboratory for Information and Decision Systems (MIT LIDS)\\$^4$Computational Science and Engineering (MIT CSE)\\
Massachusetts Institute of Technology\\
\texttt{\{caiz428,zekai\}@mit.edu}
}

\begin{document}

\ifcolmsubmission
\linenumbers
\fi

\maketitle

\begin{abstract}
While test-time scaling has advanced the performance of large language models, state-of-the-art results impose a heavy computational burden.
Better dynamic allocation of test-time compute would help mitigate this cost, but one key challenge is the lack of calibration in post-trained language models.
Here, we present \name~(\abbr), a technique for calibrating the sampling process that draws upon conformal prediction and test-time training.
Specifically, we introduce a meta-learning procedure that updates the calibration module for each input.
This allows us to provide valid confidence estimates under distributional shift, such as that in thought patterns across different stages of reasoning, or in prompt distributions between model development and deployment.
\abbr~not only provides theoretical guarantees on conformal risks, but also empirically shows higher efficiency and generalization across different reasoning tasks. At risk level $\delta{=}0.1$, \abbr~improves Qwen2.5-32B efficiency on in-distribution tasks with savings up to 47.5\% with supervised labels and 40.7\% with self-consistency labels. Under zero-shot out-of-domain settings, it improves MATH-500 savings from 24.8\% of the static calibration baseline to 67.0\% while maintaining a low empirical error rate, and the same trend holds across model families and downstream benchmarks.
Our code is publicly available at \href{https://github.com/wzekai/ORCA}{https://github.com/wzekai/ORCA}.
\end{abstract}

\section{Introduction}
Large language models have solved increasingly complex problems, such as olympiad mathematics and software engineering tasks, by vastly scaling test-time compute.
Strategies such as parallel sampling~\citep{qi2025learning}, sequential Monte Carlo~\citep{feng2025stepbystep}, Monte Carlo tree search~\citep{zhang2024rest}, verifier-guided sampling~\citep{yu2025scaling}, and self-consistency~\citep{xie2024calibrating,huang2025efficient} are key to eliciting advanced reasoning capabilities from LLMs, but they are often subject to efficiency and reliability bottlenecks. 
Post-trained LLMs are known to be \emph{miscalibrated} about when their intermediate reasoning states and final answers are correct~\citep{li2025miscalibrated}.
As a result, test-time scaling strategies often involve handcrafted parameters that balance sample quality and resource allocation -- e.g. the number of parallel rollouts, or the verifier's prompt -- heuristics which are vulnerable to reward hacking and distribution shift~\citep{snell2024scaling}.
This work addresses these challenges through a principled approach.
We aim to (i) support \emph{adaptive compute allocation} based on task difficulty, while providing \emph{statistical guarantees} on sample quality and efficiency at test time, and (ii) remain robust under \emph{distribution shift}, as the prompt distribution at deployment may differ from that of model development.

Conformal prediction~\citep{shafer2008tutorial,angelopoulos2021gentle} and calibration methods provide finite-sample coverage guarantees, and they may provide confidence estimates of whether a set of LLM outputs contains the correct answer.
In the context of test-time scaling, these methods are used to limit the number of tokens or examples that must be sampled, while still ensuring a high quality response.
For example, \emph{Thought Calibration}~\citep{wu2025thoughtcal} among others~\citep{xie2025statistical,wang2026conformal} formulates inference efficiency as a risk-controlled stopping problem, and it produces calibrated thresholds of a probe that controls risks with early stopping.
However, these methods assume a fixed inference procedure, whose dynamics are uniform over the course of sampling.
That is, they do not address the validity of confidence estimates under distribution shifts at two levels. On the sample level, reasoning patterns may vary across different positions within a long chain-of-thought (CoT).
On the dataset level, models are often deployed in out-of-distribution (OOD) circumstances that have not been seen during training.

In an orthogonal direction, \emph{Test-Time Training (TTT)}~\citep{sun2020test,sun2024ttt} provides a natural mechanism for online adaptation of LLMs.
At a high level, the goal is to adapt model weights at \emph{inference} time, based on characteristics of the input.
Specifically, an \emph{inner loop} updates a small set of ``fast'' weights based on each incoming token or reasoning step, by minimizing a self-supervised objective.
A separate \emph{outer loop} is trained across many sequences to learn the shared initialization and feature mappings, which make the inner-loop updates stable and transferable~\citep{sun2020test}.
The TTT framework is designed to improve overall modeling capability and generalization in new domains or long sequences through meta-learning.
However, it is usually applied with a reconstruction-based loss at inference time, which is not directly aligned with calibration or risk control.

In this paper, we propose \textbf{O}nline \textbf{R}easoning \textbf{Ca}libration (\textbf{\abbr}), which achieves efficient and confident test-time scaling, by framing calibration itself as an objective that can be optimized at inference time.
Concretely, the inner loop optimizes a scoring function (correctness or consistency of LLM attempt), which is implemented through a TTT layer and updated online during the reasoning or search trajectory.
As a result, calibration can be adapted to different stages of reasoning at the instance level.
In parallel, the outer-loop is a meta-training procedure that learns the shared slow weights (initialization and feature mappings of the calibration layer), so that the online updates can remain stable, data-efficient, and transferable at dataset level.
This design addresses two key limitations of prior work.
First, it preserves statistical validity by calibrating the complete algorithm executed at deployment.
Second, it improves robustness to distribution shift, by allowing instance-wise online adaptation of confidence estimates, while the base LLM remains fixed.

Empirically, our approach provides reliable stopping and candidate selection with controlled risk, reducing unnecessary test-time compute on easy instances and scaling compute only when uncertainty warrants it. At the target risk level $\delta{=}0.1$ on the in-distribution test split, applying \abbr~to Qwen2.5-32B leads to savings up to 47.5\% in supervised mode and 40.7\% in consistency mode. In the zero-shot OOD setting, our method improves MATH-500 savings from 24.8\% of static baseline to 67.0\% with supervised labels, and it consistently outperforms static probes across model families (Qwen2.5-32B, QwQ-32B, and Llama-3.3-70B) and benchmarks (MATH, GPQA, and AIME).

\section{Preliminary}
\label{sec:prelim}

\textbf{Test-time training.~}
Language models can be brittle under instance-level variation and distribution shift~\citep{lu-etal-2022-fantastically,tang-etal-2024-found}.
Since it is intractable to enumerate new use cases and finetune models, test-time training (TTT) addresses this issue by meta-training to adapt lightweight modules during inference~\citep{sun2020test}.
Concretely, TTT introduces fast weights $W_t$ updated within each test sequence by minimizing a self-supervised or proxy objective. Following the standard formulation as in \citep{sun2024ttt},
\begin{align}
\ell(W;x_t) &= \|f(\theta_K x_t;W)-C_t\|_2^2, \\
W_t &= W_{t-1}-\eta G_t, \quad G_t\approx\nabla_W \ell(W_{t-1};x_t), \\
z_t &= f(\theta_Q x_t;W_t).
\end{align}
Here, $C_t$ is the designed objective, which is a projection $\theta_V x_t$ aiming for self-reconstruction in the original formulation of \citet{sun2024ttt}, but can actually be any properly defined target (e.g., calibration label in our setting); the \emph{inner loop} performs per-step online adaptation via \emph{fast weights} $W_t$, while the \emph{outer loop} meta-learns shared \emph{slow weights} parameters (e.g., optional projections $\theta_{Q,K,V}$, initialization $W_0$) across many sequences so that inner-loop updates are stable and transferable.
For simplicity, we use $\theta_{Q,K,V}$ as a unified notation for outer-loop view parameters: they instantiate Q/K/V-style projections as feature extractors if learnable, or recover the \emph{no-QK} special case when set to identity.

TTT can be viewed as bilevel optimization: the outer objective learns how to make fast online learning effective, and the inner objective realizes instance-specific adaptation at inference time. In our setting, this mechanism is used to improve calibration quality while preserving downstream risk control through conformal calibration of the deployed procedure.

\textbf{Conformal prediction and risk control.~}
In real-world deployment, we need uncertainty estimates that make stopping/selection decisions statistically reliable rather than heuristic. Conformal prediction provides finite-sample, distribution-free validity under exchangeability by converting nonconformity scores into calibrated prediction sets~\citep{shafer2008tutorial,angelopoulos2021gentle}. In split conformal prediction, the score model is fit on training data and the decision threshold is calibrated on a held-out calibration set~\citep{vovk2005algorithmic,papadopoulos2008inductive}. Concretely, given calibration scores $\{u_i\}_{i=1}^n$ and miscoverage level $\epsilon\in(0,1)$, define the empirical quantile
\begin{equation}
\tau_{1-\epsilon}=\operatorname{Quantile}_{\lceil (n+1)(1-\epsilon)\rceil/(n+1)}(u_1,\ldots,u_n).
\end{equation}
At test time, we accept outputs whose score is below (or confidence above) this threshold to control the risk (formally defined in \Cref{sec:theory}).
Equivalently, we form a conformal set containing all candidates that satisfy the calibrated criterion.
This yields marginal finite-sample coverage (equivalently, risk control), typically of the form $\mathbb{P}(Y\in\widehat{\mathcal{C}}(X))\ge 1-\epsilon$.

Learn-then-Test (LTT) complements conformal prediction by calibrating \emph{decision rules} rather than prediction sets. Given a candidate family of rules (or thresholds), LTT tests mean-risk nulls of the form $H_j:\, r_j \ge \delta$ with valid p-values on a held-out calibration set, then applies multiple testing (e.g., fixed-sequence testing) to select a rule with finite-sample guarantee $\mathbb{P}(r(\hat{j})\le \delta)\ge 1-\epsilon$~\citep{angelopoulos2021learn,wu2025thoughtcal,wang2026conformal}. Thus, conformal and LTT share the same exchangeability-based validity principle, but target different objects: set coverage vs. rule-level risk control.

\section{\name}
\label{sec:ttt_thought_calibration}

\subsection{Setup}

\begin{table}[t]
\begin{center}
\small
\vspace{-2pt}
\caption{Notation summary.}
\label{tab:notation}
\begin{tabular}{lcl}
\toprule
Symbol & Learned in & Description \\
\midrule
$\phi_t$ & -- & Step embedding: mean-pooled LLM hidden state at step $t$ \\
$C_t$ & -- & Step label: correctness or consistency; $C_t{=}0$ at inference \\
$W_0, b_0$ & outer loop & Initial probe weights, learned via meta-training \\
$W_t, b_t$ & inner loop & Probe weights at step $t$, updated online during inference \\
$\theta_{Q,K}$ & outer loop & Unified outer view parameters (optional learned projections or identity) \\
$\eta$ & outer loop & Inner learning rate; fixed or learned \\
$s_t$ & -- & Probe score: $s_t = \sigma\!\big(f(\theta_Q \phi_t;\, W_t)\big)$ \\
$\ell(W; \phi_t)$ & -- & Inner-loop objective: $(s_t - C_t)^2$ (Brier score) \\
$\lambda^*$ & calibration & LTT-calibrated stopping threshold \\
$\delta$ & -- & Target risk upper bound \\
$\epsilon$ & -- & LTT failure probability level \\
\bottomrule
\end{tabular}
\end{center}
\vspace{-3pt}
\end{table}
We study the test-time reasoning procedure for input $x\in\mathcal{X}$.
At reasoning step $t$, the model has generated thought prefix $y_t=[y^{(1)},\ldots,y^{(t)}]$, where the current answer candidate $\operatorname{ans}(y_t)$ can be derived from the reasoning state.
We denote by $\phi_t\in\mathbb{R}^{d_\phi}$ a hidden representation extracted from the base LLM at step $t$.
This feature vector is input to the calibration module, which then outputs a probe score $z_t=f(\phi_t;W_t)$, where $W_t$ are step-dependent fast weights. For risk control, let $L(\hat y,y)\in\{0,1\}$ denote the final-decision loss (e.g., a premature stop or an incorrect accepted answer).

Our formulation follows the TTT update structure of \citet{sun2024ttt} and the Learn-then-Test (LTT) risk-control framework of \citet{angelopoulos2021learn, wu2025thoughtcal}. TTT updates use inner-loop fast-weight adaptation with outer-loop learned mappings/initialization, and risk control is enforced through LTT calibration of the deployed decision procedure.
We redefine the per-step objective to target calibration quality rather than reconstruction.
Since inference includes online updates of $W_t$, we calibrate the full adaptive algorithm (including sampling, updates, stopping rule) rather than relying on a static scorer.
Table~\ref{tab:notation} summarizes the notation correspondence.

\begin{figure}[t]
\centering
\includegraphics[width=\textwidth]{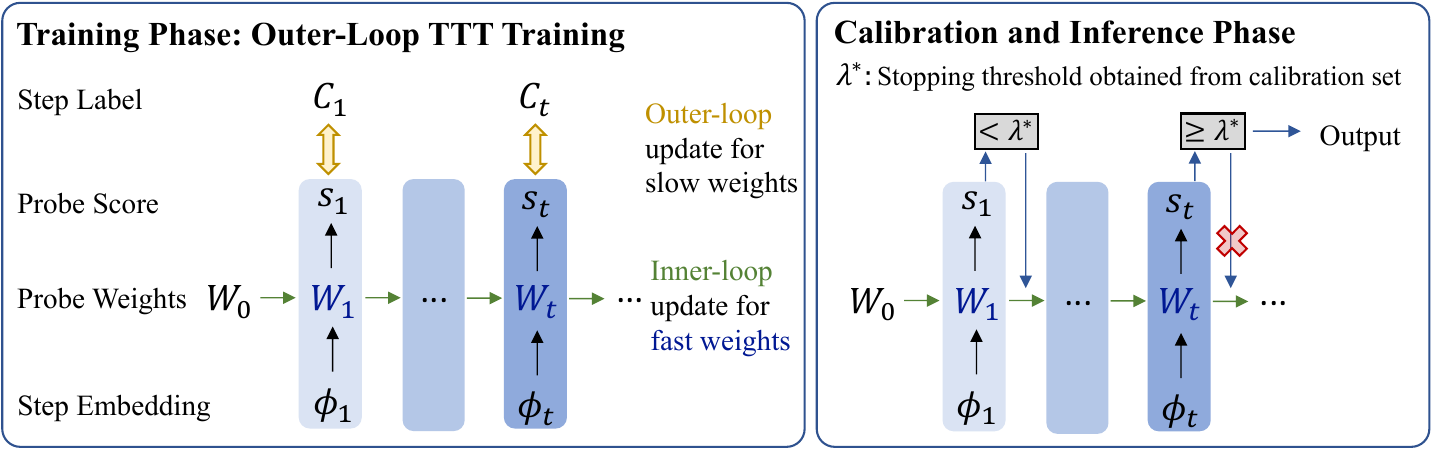}
\vspace{-0.5cm}
\caption{Framework of Online Reasoning Calibration (ORCA).}
\label{fig:orca}
\vspace{-0.25cm}
\end{figure}

\subsection{Inner loop: online-adaptive probe via fast-weight updates}
\label{sec:ttt_calibration_layer}

At each reasoning step $t$, the base LLM produces a hidden representation $\phi_t \in \mathbb{R}^{d_\phi}$ (e.g., the mean-pooled last-layer hidden state).
Our probe maintains fast weights $W_t$ that are updated online along the reasoning chain, producing a conformal score $s_t \in [0,1]$ at each step.
The procedure follows a \emph{score-then-update} protocol: the probe first scores the current step using the accumulated weights, then updates its weights before moving to the next step.

We first introduce a \textbf{vanilla form of inner loop updates}.
Let $f(\cdot\,; W): \mathbb{R}^{d_\phi} \to [0,1]$ be a probe model (e.g., $f(u; W) = \sigma(W \cdot u + b)$). Three operations at each step are as follows,

\textit{Score} (using weights from previous steps):
\begin{equation}
\label{eq:score}
s_t = f(\phi_t;\, W_{t-1})
\end{equation}

\textit{Inner-loop loss} (Brier score against label $C_t$):
\begin{equation}
\label{eq:inner_loss}
\ell(W_{t-1};\, \phi_t) = \big(s_t - C_t\big)^2
\end{equation}

\textit{Weight update} (online gradient descent):
\begin{equation}
\label{eq:update}
W_t = W_{t-1} - \eta \,\nabla_W \ell(W_{t-1};\, \phi_t)
\end{equation}
Here $C_t \in \{0,1\}$ is a step-level label indicating the quality of current answer attempt, depending on the calibration metric.
As discussed in \citet{wu2025thoughtcal}, $C_t$ can come from multiple sources in meta-training:
(i)~\emph{supervised}: $C_t = \mathbb{I}\{\operatorname{ans}(y_t) \text{ is correct}\}$, requiring ground-truth answers;
(ii)~\emph{consistent}: $C_t = \mathbb{I}\{\operatorname{ans}(y_t) = \operatorname{ans}(y_T)\}$, comparing to the full-budget answer without ground truth; or
(iii)~teacher/verifier labels from an external model.
The conformal score $s_t$ aims to evaluate the confidence of a positive label, and the probing function $f$ can often be near-linear as observed in \citep{park2023linear,zhang2025reasoning}.

\subsection{Outer loop: meta-learning for generalizable adaptive calibration}

Instance-wise updates alone may overfit local noise and degrade risk control under distribution shift. Analogous to meta-learning, the outer loop learns proper initialization and feature interfaces, so that inner-loop adaptation remains transferable across tasks~\citep{sun2020test}.
Concretely, the inner loop \emph{learns to calibrate} (task-specific online adaptation), while the outer loop \emph{learns to learn to calibrate}.
By meta-training across heterogeneous datasets and prompts, the model learns initialization and update dynamics that generalize beyond single training distribution, enabling robust and efficient calibration under distribution shift.

\begin{algorithm}[t]
\caption{Training phase: outer-loop TTT training}
\label{alg:outer}
\begin{algorithmic}[1]
\REQUIRE Training prompts $\mathcal{D}_{\mathrm{train}}$; outer parameters $\Theta_{\mathrm{outer}}=(\theta_{Q,K},W_0,\eta)$.
\FOR{each prompt $x\in\mathcal{D}_{\mathrm{train}}$}
 \STATE Unroll inner updates along the reasoning trajectory to obtain $\{W_t\}_{t=1}^{T}$ using $W_t=W_{t-1}-\eta\nabla_W\ell(W_{t-1};\phi_t)$.
 \STATE Compute $\mathcal{L}_{\mathrm{outer}}\!\left(x,\{W_t\}_{t=1}^{T};\Theta_{\mathrm{outer}}\right)=\sum_{t=1}^{T}(s_t-C_t^{\text{true}})^2$.
 \STATE Update trainable components of $\Theta_{\mathrm{outer}}$ by differentiating $\mathcal{L}_{\mathrm{outer}}$ through the unroll.
\ENDFOR
\STATE \textbf{Return} trained $\Theta_{\mathrm{outer}}$.
\end{algorithmic}
\end{algorithm}

Equations~\eqref{eq:score}--\eqref{eq:update} define the simplest TTT update rule, where a probe with fast weights $W \in \mathbb{R}^{1 \times d_\phi}$ operates directly on the embedding $\phi_t$.
This parameterization has only $d_\phi + 1$ learnable parameters (the initialization $W_0$ and bias $b_0$) and can be viewed as an online-adaptive logistic regression, which is denoted as the vanilla or \emph{no-QK} variant.

Analogous to \cite{sun2024ttt}, a natural extension introduces learned projections $\theta_K, \theta_Q \in \mathbb{R}^{d_h \times d_\phi}$ that map $\phi_t$ to a lower-dimensional space before the update and scoring operations:
\begin{align}
s_t &= f(\theta_Q \phi_t;\, W_{t-1}), \quad W \in \mathbb{R}^{1 \times d_h} \label{eq:score_kq}\\
\ell(W;\, \phi_t) &= \big(f(\theta_K \phi_t;\, W) - C_t\big)^2 \label{eq:loss_kq}
\end{align}
This \emph{QK} variant allows the update direction ($\theta_K$) and scoring direction ($\theta_Q$) to attend to different aspects of the hidden state.
The projections $\theta_K, \theta_Q$ are "slow weights" and are learned in the outer loop along with $W_0$ and $\eta$.
Both variants share the same single-step expressiveness, but they differ in the dynamics of online adaptation.
The no-QK variant updates in the full $d_\phi$-dimensional space, while Q/K updates are constrained to a $d_h$-dimensional subspace. For comprehensiveness, we experiment with TTT with and without QK updates, and observe significant improvement over static baselines for both variants.

We now formally introduce the \textbf{general framework of slow weight training}. Let $\Theta_{\mathrm{outer}}=(\theta_{Q,K},W_0,\eta)$ denote outer parameters. Recall that in our framework, base LLM parameters are not updated during training by default. Given a training prompt $x$, we first unroll the inner-loop updates along its trajectory to obtain $\{W_t\}_{t=1}^{T}$, and then optimize:
\begin{equation}
\min_{\Theta_{\mathrm{outer}}}\;\mathbb{E}_{x\sim\mathcal{D}_{\mathrm{train}}}\,\mathcal{L}_{\mathrm{outer}}\!\left(x,\{W_t\}_{t=1}^{T};\Theta_{\mathrm{outer}}\right),
\end{equation}
with the outer loss defined as
\begin{equation}
\mathcal{L}_{\mathrm{outer}}\!\left(x,\{W_t\}_{t=1}^{T};\Theta_{\mathrm{outer}}\right) := \sum_{t=1}^{T}\left(s_t - C_t^{\text{true}}\right)^2.
\end{equation}
subject to $W_t=W_{t-1}-\eta\nabla_W\ell(W_{t-1};x_t)$. We optimize this bilevel objective with truncated backpropagation through inner updates, then run LTT on a held-out calibration split produced by the same deployed procedure to select stopping thresholds.
\Cref{alg:outer} summarizes the general form of outer-loop training algorithm, which is visualized in \Cref{fig:orca}.

\newpage
\subsection{Risk-controlled conformal reasoning with online self-calibration}\label{sec:thought_calibration_alignment}

\begin{algorithm}[t]
\caption{Calibration and inference phase}
\label{alg:inner}
\begin{algorithmic}[1]
\REQUIRE Calibration prompts $\mathcal{D}_{\mathrm{cal}}$; threshold grid $\Lambda=\{\lambda_1>\cdots>\lambda_m\}$; prompt $x$; trained $\Theta_{\mathrm{outer}}$.
\STATE \textbf{(A) Calibrate stopping threshold via LTT.}
\FOR{each $\lambda_j\in\Lambda$}
 \STATE Run deployed procedure on $\mathcal{D}_{\mathrm{cal}}$ with threshold $\lambda_j$; compute $\widehat{R}_n(\lambda_j)$ and p-value $p_j$.
\ENDFOR
\STATE Apply fixed-sequence testing over $\{H_j\}_{j=1}^{m}$ to control FWER, obtain $\Lambda_{\mathrm{valid}}$, and choose the most aggressive $\lambda^\star\in\Lambda_{\mathrm{valid}}$.
\STATE \textbf{(B) Deploy inference procedure with online self-calibration.}
\STATE Initialize fast weights $W\leftarrow W_0$.
\FOR{$t=1,\ldots,T$}
 \STATE Obtain the current hidden representation $\phi_t$ from the reasoning state.
 \STATE Compute probe score $s_t=f(\phi_t;W)$.
 \IF{$s_t\ge\lambda^\star$}
  \STATE \textbf{Stop} and output $\hat z\leftarrow\operatorname{ans}(y_t)$.
  \STATE \textbf{return} $\hat z$.
 \ENDIF
 \STATE Set pseudo-target $C_t\leftarrow 0$ and perform inner-loop update $W\leftarrow W-\eta\nabla_W\ell(W;\phi_t)$.
\ENDFOR
\STATE \textbf{Return} $\operatorname{ans}(y_T)$ if the budget is exhausted.
\end{algorithmic}
\end{algorithm}

Following \citet{wu2025thoughtcal}, calibration is performed on a stopping \emph{decision rule} (i.e. the full deployed procedure), rather than on raw uncalibrated scores. 
For a threshold $\lambda\in\Lambda$, recall the score process produced by the deployed procedure $s_t(x)\;:=\;f(\phi_t;\,W_{t-1})$.
Define the stopping time
\begin{equation}
\tau_\lambda(x)\;:=\;\min\{t\le T:\; s_t(x)\ge \lambda\},
\end{equation}
and the corresponding deployed decision rule / procedure output
\begin{equation}
\mathcal{A}_\lambda(x)\;:=\;\operatorname{ans}\!\big(y_{\tau_\lambda(x)}\big),
\end{equation}
i.e., the final answer attempt at the stopping time (or $\operatorname{ans}(y_T)$ if the budget is exhausted). 
Learn-then-Test (LTT) selects $\lambda^\star$ by calibrating the \emph{risk of the entire procedure} $\mathcal{A}_\lambda$ on a held-out calibration set.
In particular, LTT sweeps a grid $\Lambda = \{\lambda_1 > \cdots > \lambda_m\}$ from conservative to aggressive, testing at each $\lambda$ whether the empirical risk exceeds the tolerance using a binomial p-value, i.e., we test
\begin{equation}
H_j:\;\mathbb{E}[R(y_{\tau_{\lambda_j}})]\ge\delta,
\end{equation}
where $\delta$ is the target risk upper bound, and construct p-values
\begin{equation}
p_{j}^{\mathrm{BT}}:=\mathbb{P}(\mathrm{Binom}(n,\delta)\le n\widehat{R}_n(\lambda_j)).
\end{equation}
where $\epsilon$ is the target failure probability level for the test family.
Apply fixed-sequence testing over $\{H_j\}_{j=1}^m$ to control family-wise error rate (FWER), then select the most aggressive rejected threshold $\lambda^*$. The selected threshold satisfies
\begin{equation}
\mathbb{P}\!\left(\mathbb{E}[R(y_{\tau_{\lambda^*}})]\le\delta\right)\ge 1-\epsilon.
\end{equation}
Consequently, FWER control yields finite-sample risk control at level $(\delta,\epsilon)$ under exchangeability.
The guarantees apply to the \emph{entire deployed procedure}, including reasoning chain expansion, online fast-weight updates, and threshold-based stopping~(\Cref{fig:orca}).
The full calibration and inference-time deployment of \abbr~is presented in \Cref{alg:inner}. We detail the theoretical guarantees in Appendix~\ref{sec:theory} and more discussions in Appendix~\ref{sec:discussion}.

\section{Experiments}
\label{sec:experiments}

We evaluate the \name~framework on reasoning efficiency: given a risk tolerance $\delta$, how much compute can be saved by early stopping while maintaining answer quality?
We compare against the static linear probe of \citet{wu2025thoughtcal} across multiple models, label modes, and out-of-distribution benchmarks. More experimental results and ablation studies are deferred to Appendix~\ref{sec:exp_appendix}.

\subsection{Setup}
\label{sec:exp_setup}

\textbf{Datasets.~}
We construct a \textbf{5K training corpus} by combining three sources:
(i)~the s1K dataset of 1{,}000 math problems from \citet{muennighoff2025s1},
(ii)~2{,}000 problems from OpenR1~\citep{openr1},
and (iii)~2{,}000 problems from DeepMath~\citep{deepmath}.
Problems are split 3:1:1 into training (3{,}000), calibration (1{,}000), and test (1{,}000) sets.
For out-of-distribution~(OOD) evaluation, we use five held-out benchmarks: MATH-500~\citep{hendrycks2021math}, GPQA-Diamond~\citep{rein2024gpqa} (198 problems), and AIME 2024/2025/2026 (30 problems each).
No OOD problems appear in training or calibration sets.

\textbf{Models.~}
Our primary experiments use \textbf{Qwen2.5-32B-Instruct}~\citep{qwen25}, extracting mean-pooled last-layer hidden states ($d_\phi = 5{,}120$) at each reasoning step.
We also evaluate on \textbf{QwQ-32B}~\citep{qwq32b} and \textbf{Llama-3.3-70B-Instruct}~\citep{llama3} ($d_\phi = 8{,}192$) to test cross-model generalization.
Reasoning trajectories are generated by DeepSeek-R1-671B~\citep{deepseekr1}, and step labels are produced by a teacher model (Qwen-3-32B for correctness, GPT-4.1 for evaluation).

\textbf{Step definition and label modes.~}
To separate a reasoning trajectory into individual steps, we completely follow \citet{wu2025thoughtcal}: we use sections delimited by \texttt{\textbackslash n\textbackslash n}, which also contain "wait" or "but". This separation is reasonable for those models after instruction tuning and RL (\cite{muennighoff2025s1}).
We evaluate two labeling strategies yielding two calibration targets, following \citet{wu2025thoughtcal}:
\textbf{Supervised}, where $C_t = \mathbb{I}\{z_t \text{ is correct}\}$ requires ground-truth labels, and
\textbf{Consistent}, where $C_t = \mathbb{I}\{z_t = z_T\}$ compares intermediate answers to the full-budget answer (no labels required).

\textbf{Metrics.~}
At each risk tolerance $\delta$, Learn-then-Test (LTT; \citet{angelopoulos2021learn}) calibration determines a threshold $\lambda^*$.
We report two metrics.
\ding{182}~\textbf{Savings} $= 1 - \bar{t}_{\text{stop}} / \bar{t}_{\text{total}}$, the fraction of reasoning steps saved by early stopping. We verify in \Cref{sec:step_vs_token} that step-level and token-level savings are highly consistent, so we report step-level savings throughout.
\ding{183}~\textbf{Error rate}, the fraction of problems where the model is stopped at a step where its answer is still incorrect.
Since step labels are cumulative (flip after first correct attempt), only stopping too early leads to an error.
Following \citet{wu2025thoughtcal}, we exclude problems on which the base model never produces a correct answer in its full rollout: on such problems the final answer is wrong regardless of where we stop, so they are uninformative for evaluating early stopping.
These two metrics form a natural trade-off controlled by the threshold $\lambda^*$: a lower threshold stops earlier, yielding higher savings but higher error risk; a higher threshold is more conservative, with lower error but less compute saved.
LTT selects $\lambda^*$ to satisfy the guarantee $\mathbb{P}(R \leq \delta) \geq 1-\epsilon$, where we fix $\epsilon=0.05$ and sweep the risk tolerance $\delta$.
Unless stated otherwise, we report results at $\delta = 0.1$.

\textbf{Training and epoch selection.~}
All TTT-Probe variants are meta-trained with Adam (outer lr $= 10^{-3}$), gradient clipping at 1.0, and inner learning rate $\eta = 0.01$.
We select \textbf{epoch 20} for the no-QK variant and \textbf{epoch 10} for all QK variants (see \Cref{sec:epoch} for details).
Score trajectories are smoothed with a rolling window of 10 steps.

\begin{table}[ht]
\centering
\small
\caption{In-distribution early-stopping performance on the 5K test set (Qwen2.5-32B, $\epsilon{=}0.05$). TTT-Probe (no-QK) improves savings by 24.9\% relative over the static baseline.}
\label{tab:main}
\vspace{1pt}
\begin{tabular}{l cc cc cc cc}
\toprule
& \multicolumn{2}{c}{$\delta=0.05$} & \multicolumn{2}{c}{$\delta=0.1$} & \multicolumn{2}{c}{$\delta=0.15$} & \multicolumn{2}{c}{$\delta=0.2$} \\
\cmidrule(lr){2-3} \cmidrule(lr){4-5} \cmidrule(lr){6-7} \cmidrule(lr){8-9}
Method & Sav. & Err. & Sav. & Err. & Sav. & Err. & Sav. & Err. \\
\midrule
\multicolumn{9}{l}{\textit{Supervised labels}} \\
Static Probe & .220 & .055 & .380 & .105 & .512 & .159 & .625 & .208 \\
TTT no-QK & \textbf{.282} & .053 & \textbf{.475} & .110 & \textbf{.575} & .152 & .673 & .192 \\
TTT QK ($d_h{=}128$) & .233 & .046 & .414 & .103 & .560 & .150 & \textbf{.674} & .204 \\
\midrule
\multicolumn{9}{l}{\textit{Consistent labels (no ground truth)}} \\
Static Probe & .166 & .049 & .345 & .098 & .483 & .156 & .573 & .197 \\
TTT no-QK & .220 & .045 & \textbf{.407} & .096 & \textbf{.529} & .141 & \textbf{.644} & .193 \\
TTT QK ($d_h{=}128$) & \textbf{.232} & .064 & .397 & .113 & .524 & .150 & .629 & .187 \\
\bottomrule
\end{tabular}
\end{table}

\subsection{In-Distribution Results}
\label{sec:main_results}

Table~\ref{tab:main} (and \Cref{fig:risk_compute} in Appendix) compares the static baseline against TTT-Probe across four risk levels.
At $\delta = 0.1$, the supervised TTT-Probe (no-QK) saves 47.5\% of reasoning steps compared to 38.0\% for the static baseline, a \textbf{24.9\% relative improvement}.
The QK variant achieves 41.4\%, an 8.9\% relative improvement.
In consistent mode, the no-QK probe saves 40.7\% vs.\ 34.5\%, an 18.2\% relative improvement \emph{without any ground-truth labels}.

Across all four $\delta$ levels, TTT-Probe uniformly dominates the baseline.
Both the no-QK and QK variants maintain error rates within or close to the prescribed $\delta$ budgets, demonstrating that the online adaptation provides genuine calibration improvements.

\subsection{Out-of-Distribution Generalization}
\label{sec:ood}

A key motivation for TTT-Probe is robustness under distribution shift.
Table~\ref{tab:ood} evaluates probes trained on the 5K corpus and applied \emph{zero-shot} to five OOD benchmarks.

Under supervised labels, both TTT variants achieve strong OOD generalization on MATH-500: no-QK saves 63.7\% and QK saves 67.0\%, compared to 24.8\% for the baseline (a \textbf{2.6--2.7$\times$ improvement}), while keeping errors below 2.3\%.
On GPQA-Diamond, the no-QK probe achieves 71.5\% savings.
Under consistent labels, the QK variant achieves 63.7\% on MATH-500 (2.7$\times$ the baseline), demonstrating that label-free TTT is viable for OOD deployment.

\begin{table}[ht]
\centering
\small
\vspace{-1pt}
\caption{OOD generalization at $\delta{=}0.1$. The TTT-Probe achieves 2.6--2.7$\times$ the baseline savings on MATH-500 under supervised labels.}
\label{tab:ood}
\vspace{1pt}
\begin{tabular}{l cc cc cc cc cc}
\toprule
& \multicolumn{2}{c}{MATH-500} & \multicolumn{2}{c}{GPQA} & \multicolumn{2}{c}{AIME'24} & \multicolumn{2}{c}{AIME'25} & \multicolumn{2}{c}{AIME'26} \\
\cmidrule(lr){2-3} \cmidrule(lr){4-5} \cmidrule(lr){6-7} \cmidrule(lr){8-9} \cmidrule(lr){10-11}
Method & Sav. & Err. & Sav. & Err. & Sav. & Err. & Sav. & Err. & Sav. & Err. \\
\midrule
\multicolumn{11}{l}{\textit{Supervised labels}} \\
Static Probe & .248 & .008 & .643 & .270 & .158 & .050 & .139 & .000 & .147 & .050 \\
TTT no-QK & .637 & .023 & \textbf{.715} & .300 & .293 & .150 & \textbf{.265} & .056 & \textbf{.198} & .050 \\
TTT QK ($d_h{=}128$) & \textbf{.670} & .021 & .665 & .210 & \textbf{.295} & .100 & .258 & .000 & .134 & .050 \\
\midrule
\multicolumn{11}{l}{\textit{Consistent labels (no ground truth)}} \\
Static Probe & .239 & .004 & .602 & .328 & .118 & .033 & .101 & .000 & .147 & .100 \\
TTT no-QK & .555 & .012 & .598 & .318 & .141 & .033 & \textbf{.166} & .067 & \textbf{.154} & .067 \\
TTT QK ($d_h{=}128$) & \textbf{.637} & .016 & \textbf{.653} & .328 & \textbf{.185} & .033 & .139 & .000 & .092 & .000 \\
\bottomrule
\end{tabular}
\end{table}

\subsection{Cross-Model Performance}
\label{sec:cross_model}

To verify that our findings generalize beyond Qwen2.5-32B, we evaluate the same configurations on QwQ-32B ($d_\phi = 5{,}120$) and Llama-3.3-70B-Instruct ($d_\phi = 8{,}192$).
All probes are trained and evaluated independently on each model's own embeddings.

Both TTT-Probe variants consistently outperform the static baseline across all three models.
The no-QK probe achieves relative improvements of 24.9\% on Qwen, 33.7\% on QwQ, and 19.8\% on Llama.
The QK variant also improves over the baseline on all models (6.8--27.6\% relative).
All error rates remain within or close to the $\delta{=}0.1$ budget, confirming that the online adaptation mechanism is model-agnostic.

\begin{table}[ht]
\centering
\small
\vspace{-2pt}
\caption{Cross-model results ($\delta{=}0.1$, supervised). TTT-Probe consistently outperforms the static baseline across all three model families.}
\label{tab:cross_model}
\vspace{2pt}
\begin{tabular}{l cc cc cc}
\toprule
& \multicolumn{2}{c}{Qwen2.5-32B} & \multicolumn{2}{c}{QwQ-32B} & \multicolumn{2}{c}{Llama-3.3-70B} \\
\cmidrule(lr){2-3} \cmidrule(lr){4-5} \cmidrule(lr){6-7}
Method & Sav. & Err. & Sav. & Err. & Sav. & Err. \\
\midrule
Static Probe & .380 & .105 & .295 & .094 & .354 & .104 \\
TTT no-QK & \textbf{.475} & .110 & \textbf{.394} & .081 & \textbf{.424} & .090 \\
TTT QK ($d_h{=}128$) & .414 & .103 & .376 & .076 & .378 & .081 \\
\bottomrule
\end{tabular}
\end{table}

\subsection{Multiple Calibration Targets}
\label{sec:sup_vs_con}

Supervised labels yield 10--17\% higher savings than consistent labels on in-distribution data, reflecting the additional information provided by ground-truth correctness~(\Cref{tab:sup_con}).
The no-QK probe achieves 47.5\% savings in supervised mode and 40.7\% in consistent mode, both substantial improvements over the static baseline (24.9\% and 18.2\% relative respectively).
On OOD data, supervised no-QK averages 42.2\% savings vs.\ 32.3\% for consistent.
The consistent TTT-Probe remains attractive for deployment where ground-truth labels are unavailable: it achieves a 34.0\% relative improvement on OOD data over the consistent baseline while requiring only the model's own full-budget answers as supervision.

\begin{table}[ht]
\centering
\small
\vspace{-2pt}
\caption{Supervised vs.\ consistent comparison ($\delta{=}0.1$), including average OOD savings and error across five benchmarks.}
\label{tab:sup_con}
\vspace{2pt}
\begin{tabular}{l cccc cccc}
\toprule
& \multicolumn{4}{c}{Supervised} & \multicolumn{4}{c}{Consistent} \\
\cmidrule(lr){2-5} \cmidrule(lr){6-9}
& \multicolumn{2}{c}{In-dist} & \multicolumn{2}{c}{OOD avg} & \multicolumn{2}{c}{In-dist} & \multicolumn{2}{c}{OOD avg} \\
\cmidrule(lr){2-3} \cmidrule(lr){4-5} \cmidrule(lr){6-7} \cmidrule(lr){8-9}
Configuration & Sav. & Err. & Sav. & Err. & Sav. & Err. & Sav. & Err. \\
\midrule
Static Probe & .380 & .105 & .267 & .076 & .345 & .098 & .241 & .093 \\
TTT no-QK & \textbf{.475} & .110 & \textbf{.422} & .116 & \textbf{.407} & .096 & .323 & .099 \\
TTT QK ($d_h{=}128$) & .414 & .103 & .404 & .076 & .397 & .113 & \textbf{.341} & .076 \\
\bottomrule
\end{tabular}
\end{table}

\section{Related Work}
\vspace{-1pt}

\textbf{Efficient test-time scaling.~}
Recent literature span two directions.
First, numerous works focus on reducing overthinking in post-trained LLMs, either during training or at inference time. Examples include explicit stopping policies and compute-aware generation control~\citep{guo2025deepseek,sui2025stop,han2024token,hou2025thinkprune,yang2025dynamic,zhang2025reasoning,sun2025stop}.
Second, self-consistency between parallel trajectories may be used to rank or terminate reasoning attempts~\citep{wang2022self,mitchell2022enhancing,weng2023large,wang2024math}, and may also be combined with calibration methods for controllable efficiency gains~\citep{xie2024calibrating,huang2025efficient,liu2026pets}.
Our work also aims to sample efficiently from language models at inference time, but differs in its scope.
While prior works calibrate probes that predict when to stop, we conformalize the end-to-end decision rule for stopping, which includes reasoning expansion and online adaptation to sampling dynamics.

\textbf{Uncertainty quantification and calibration for LLM reasoning.~}
Uncertainty quantification methods have primarily been used to calibrate whether language model outputs are self-consistent~\citep{rubin2025conformal}, high quality~\citep{quachconformal,semanticdensity,GraphUncertaintyLM,li2025calibrating,schuster2022confident}, and factual~\citep{tatsu2024factuality,Candes24EnhancedConformal,cherian2024large,liu2024uncertainty,prinster2026conformal}.
Some recent papers extend the framework from LLMs to agentic reasoning~\citep{diverseagententropy,sadhuka2025valuator,lee2026agentic}.
However,these methods filter sets of text post-hoc, rather than guide decoding in an online setting.
More similar to this work, several methods calibrate the sampling of output sets~\citep{quachconformal,wu2025thoughtcal,huanguncertainty,xiong2025atts,xie2025statistical,wang2026conformal,huang2026cats}, but these methods often assume that the distributions of reasoning steps and prompts are static. 
Our work differs in that it models the sampling dynamics, as well as potential shifts in prompt distribution at deployment.

\textbf{Test-time training and online adaptation.~}
TTT aims to improve generalization capability under distribution shifts by learning to conduct lightweight parameter updates at test-time, usually through self-supervised loss~\citep{sun2020test}.
The framework is also widely adopted to design efficient architectures including RNNs and linear transformers~\citep{sun2024ttt,zhang2026testtime}, or adopt other designs such as sample-specific vector~\citep{hu2025slot}, LoRA~\citep{wang2024greater} or input perplexity minimization~\citep{hu2025test} to adapt language models at test time. In comparison, we are the first to introduce online adaptation and test-time training into calibration of LLM reasoning.

\section{Conclusion}
\vspace{-1pt}
We introduced \name~(\abbr), a unified framework for risk-controlled test-time scaling that combines online test-time training with conformal calibration of the deployed stopping rule.
The key idea is to treat calibration itself as an adaptive prediction problem:
the inner loop performs instance-specific \emph{learn-to-calibrate} updates, while the outer loop meta-learns the initialization and update dynamics, which transfer across datasets and remain robust under distribution shift.
By calibrating the full deployed procedure via LTT, \abbr~provides finite-sample risk control while enabling adaptive compute allocation.
Empirically, \abbr~significantly improves efficiency at controlled error rates across multiple model families and benchmarks, including challenging zero-shot OOD settings. These results indicate that dynamically updated calibration modules can substantially improve both reliability and compute efficiency relative to static confidence estimators.
More broadly, this work highlights a practical direction for integrating meta-learning and conformal decision-making towards efficient reasoning,
and demonstrates that jointly designing adaptation and calibration yield effective and robust systems.

\vspace{-1pt}
\section*{Ethics Statement}
\vspace{-1pt}
This paper is about technical methods for LLM reasoning, and is not directly related to any ethical issues. Moreover, the calibration methods can be applied to produce more trustworthy LLM outputs, which may benefit the ethical aspects in LLM usages.

\bibliography{preprint}
\bibliographystyle{colm2026_conference}

\newpage
\appendix
\section*{Appendix}

\section{Theoretical Guarantees}
\label{sec:theory}
We provide theoretical guarantees of risk control through \abbr~in this section, with slightly more complicated and rigorous notations.
The central question is whether intra-instance online updates invalidate Learn-then-Test (LTT) risk control. 
Our key observation is that LTT calibrates the \emph{entire deployed procedure} as a black-box algorithm. 
As long as (i) the procedure resets its internal state across instances and (ii) calibration and test runs are exchangeable under the same deployed procedure, finite-sample risk control remains valid.

\paragraph{Calibration and test data.}
Let $\mathcal{D}_{\mathrm{cal}}=\{(X_i,Y_i)\}_{i=1}^n$ be a calibration set and let $(X_{n+1},Y_{n+1})$ be a fresh test point.
For the p-value construction below, we assume $(X_i,Y_i)$ are i.i.d.\ from a distribution $P$ (hence exchangeable), and independent of the random seeds used by the deployed algorithm.
All trained outer-loop parameters (e.g., $\Theta_{\mathrm{outer}}$ and any fixed feature extractors) are treated as fixed constants \emph{independent} of $\mathcal{D}_{\mathrm{cal}}$ (e.g., learned on disjoint data); equivalently, all guarantees hold conditional on these fixed parameters.

\paragraph{Deployed procedure as a randomized map.}
For each threshold $\lambda\in\Lambda$, define the \emph{entire deployed reasoning procedure} as a (possibly randomized) mapping
\begin{equation}
\mathcal{A}_\lambda:\mathcal{X}\times\mathcal{U}\to\mathcal{Y},
\qquad 
\hat Y=\mathcal{A}_\lambda(x;U),
\end{equation}
where $U\sim P_U$ denotes all internal randomness (LLM sampling, search randomness, etc.).
This mapping includes (i) initializing fast weights $W_0\leftarrow \Theta_{\mathrm{outer}}$, (ii) computing step representations $\phi_t$, (iii) computing scores $s_t=f(\phi_t;W_{t-1})$, (iv) updating fast weights within the instance, and (v) stopping at $\tau_\lambda(x;U)=\min\{t\le T:\, s_t(x;U)\ge\lambda\}$.

\paragraph{Risk definition (includes algorithm randomness).}
Let $L:\mathcal{Y}\times\mathcal{Y}\to\{0,1\}$ be a 0--1 risk (e.g., incorrectness indicator). 
For a single instance $(X,Y)$ and randomness $U$, define
\begin{equation}
R(\lambda;X,Y,U)\;:=\;L(\mathcal{A}_\lambda(X;U),Y)\in\{0,1\}.
\end{equation}
Let the \emph{marginal (deployment) risk} of threshold $\lambda$ be
\begin{equation}
r(\lambda)\;:=\;\mathbb{E}_{(X,Y)\sim P}\,\mathbb{E}_{U\sim P_U}\big[\,R(\lambda;X,Y,U)\,\big].
\end{equation}
On the calibration set, for each $\lambda$ we run the deployed procedure independently per instance with fresh randomness $U_i\stackrel{\mathrm{i.i.d.}}{\sim}P_U$ and obtain
\begin{equation}
R_i(\lambda)\;:=\;R(\lambda;X_i,Y_i,U_i),\qquad i=1,\dots,n.
\end{equation}

\begin{lemma}[Intra-instance adaptation preserves inter-instance exchangeability]
\label{lemma:exchangeability}
Fix any threshold $\lambda\in\Lambda$. 
If $\{(X_i,Y_i)\}_{i=1}^{n+1}$ are exchangeable and $\{U_i\}_{i=1}^{n+1}$ are i.i.d.\ and independent of $\{(X_i,Y_i)\}$, then the sequence
$\{R_i(\lambda)\}_{i=1}^{n+1}$ with $R_i(\lambda):=R(\lambda;X_i,Y_i,U_i)$ is exchangeable (indeed i.i.d.\ under the i.i.d.\ assumption).
\end{lemma}

\begin{proof}
Because the deployed procedure resets its internal state at the start of each instance, $(X_i,Y_i,U_i)\mapsto R(\lambda;X_i,Y_i,U_i)$ is the same measurable function applied separately to each triple.
Exchangeability of $\{(X_i,Y_i,U_i)\}$ implies exchangeability of $\{R_i(\lambda)\}$ by closure of exchangeability under measurable mappings.
\end{proof}

\begin{theorem}[Finite-sample risk control of \abbr{} via LTT fixed-sequence testing]
\label{thm:risk_control}
Let $\delta\in(0,1)$ be a target risk level and $\epsilon\in(0,1)$ a failure-probability level.
Fix an ordered threshold grid $\Lambda=\{\lambda_1>\lambda_2>\cdots>\lambda_m\}$.
For each $\lambda_j$, consider the null hypothesis
\begin{equation}
H_j:\; r(\lambda_j)\ge \delta
\quad (\text{equivalently, the procedure is \emph{not} $\delta$-safe}).
\end{equation}
For each $j$, let $\widehat{r}_n(\lambda_j)=\frac{1}{n}\sum_{i=1}^n R_i(\lambda_j)$ and define the one-sided binomial p-value
\begin{equation}
p_j \;:=\; \mathbb{P}\!\left(\mathrm{Binom}(n,\delta)\le n\,\widehat{r}_n(\lambda_j)\right).
\end{equation}
Assume $\{R_i(\lambda_j)\}_{i=1}^n$ are i.i.d.\ Bernoulli with mean $r(\lambda_j)$ (which holds under the assumptions above).
Apply fixed-sequence testing (FST): test $H_1,H_2,\dots$ in order at level $\epsilon$, rejecting $H_j$ if $p_j\le \epsilon$, and stop at the first $j$ such that $p_j>\epsilon$.
Let $\hat{\jmath}$ be the last rejected index (or $\hat{\jmath}=0$ if none are rejected) and output $\lambda^\star:=\lambda_{\hat{\jmath}}$ (the most aggressive rejected threshold).
Then
\begin{equation}
\mathbb{P}_{\mathcal{D}_{\mathrm{cal}}}\big(r(\lambda^\star)\le \delta\big)\;\ge\;1-\epsilon.
\end{equation}
\end{theorem}

\begin{proof}
Under $H_j$ (i.e., $r(\lambda_j)\ge\delta$), the p-value $p_j$ is super-uniform: for any $\alpha\in[0,1]$,
\begin{equation}
\mathbb{P}(p_j\le \alpha \mid H_j)\le \alpha
\end{equation}
(standard one-sided binomial test; see \citet{angelopoulos2021learn}).
In FST, if any true null is ever rejected, then the \emph{first} true null in the ordered sequence must be rejected at level $\epsilon$.
Therefore, by super-uniformity,
\begin{equation}
\mathbb{P}(\text{FST rejects any true }H_j)\;\le\;\epsilon,
\end{equation}
i.e., FST controls the family-wise error rate (FWER) at level $\epsilon$.
On the complement event (probability at least $1-\epsilon$), every rejected $H_j$ is false, hence $r(\lambda_j)<\delta$ for all rejected indices, and in particular for the selected $\lambda^\star$ (the most aggressive rejected threshold). Thus $\mathbb{P}(r(\lambda^\star)\le \delta)\ge 1-\epsilon$.
\end{proof}

\begin{remark}[Marginal guarantee]
    \Cref{thm:risk_control} provides marginal guarantee (expectation of all prompt distribution), not conditional guarantee (specific for a certain prompt).
\end{remark}

\begin{remark}[General bounded risks]
    If one wishes to allow $L(\cdot,\cdot)\in[0,1]$ (not necessarily $\{0,1\}$), the same proof structure holds by replacing the binomial p-values with any valid super-uniform p-values for the mean-risk null $r(\lambda_j)\ge\delta$ (e.g., Hoeffding-type tests under independence), without changing the FST argument.
\end{remark}

\section{Discussions}\label{sec:discussion}

\paragraph{Self-supervised inner loop at inference.}
A central design choice is that, during deployment (both calibration-set runs and test inference), the inner-loop update uses $C_t=0$ at every non-stopping step.
Formally, updates are applied only while $s_t < \lambda^*$; once $s_t \ge \lambda^*$, the procedure stops and no further gradient step is taken.
This yields a fully self-supervised inference-time update rule, since no external label is required after deployment.
It also improves \textbf{training--inference consistency}: in meta-training, all pre-transition steps (for which $C_t^{\text{true}}=0$) follow the same inner-loop dynamics as inference, while supervision enters only through the outer objective.

\paragraph{Detecting the reasoning breakthrough.}
\label{sec:transition}

For each problem at training time, we assume (after cumulative transformation) that the true step labels $C_t^{\text{true}}$ follow a monotone binary sequence, $[0,0,\ldots,0,1,1,\ldots,1]$, with a single transition point at which the answer first becomes correct.
The outer-loop objective $\sum_t (s_t-C_t)^2$ therefore encourages the probe to localize this transition, assigning low scores before the transition and high scores afterward.

At inference time, we \emph{do not} assume access to $C_t^{\text{true}}$ and we do \emph{not} claim that
``not stopping'' implies $C_t^{\text{true}}=0$.
Instead, we use the pseudo-target $C_t$ to adapt the fast weights toward the instance-specific pre-transition baseline, which empirically acts as a novelty detector for the transition. 
Consequently, the probe accumulates an instance-specific representation of typical pre-transition reasoning patterns.
When a substantive reasoning breakthrough occurs and the embedding $\phi_t$ shifts accordingly, the probe score increases because the new state is no longer well explained by the current adaptation.
Meta-training embeds this behavior into the slow parameters: $W_0$ sets the initialization and $\eta$ controls the adaptation rate to the problem-specific baseline.
Under this view, the $C_t{=}0$ inner loop acts as an implicit novelty detector, where ``novelty'' corresponds to the transition pattern learned during meta-training.
From another perspective, the entire deployment procedure can be treated as a \emph{binary classification task}, hence there is no train-test gap at all.

\paragraph{Online adaptation vs.\ on-policy validity.}
Our method is \emph{online}: during a single reasoning trajectory, the fast weights are updated step-by-step, and the stopping decision is made from the current adapted state.
This describes the \emph{algorithmic form} of the method and does not, by itself, imply a distributional guarantee.
By contrast, \emph{on-policy} is the condition required by conformal/LTT calibration: calibration trajectories and deployment trajectories should be generated under the same effective policy (including search behavior, stopping logic, and update dynamics).
Therefore, ORCA is best described as an online method with on-policy validity guarantees.
If the deployment policy changes, one should re-calibrate under the new policy or apply explicit off-policy correction (e.g., importance weighting); otherwise, exchangeability assumptions can fail.

\section{Additional Experiments}\label{sec:exp_appendix}

\subsection{Ablation Studies}
\label{sec:ablation}

\begin{figure}[ht]
\centering
\vspace{-1pt}
\includegraphics[width=\textwidth]{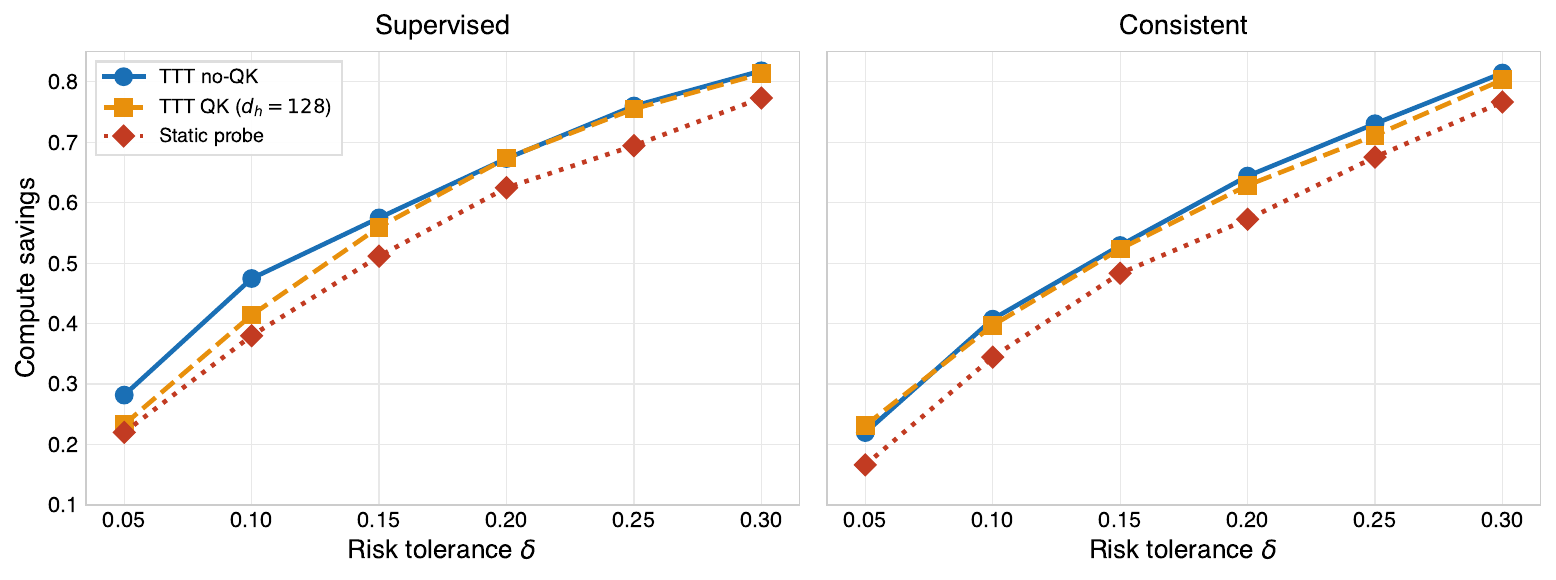}
\vspace{-4pt}
\caption{Compute savings vs.\ risk tolerance $\delta$ for supervised (left) and consistent (right) labels (Qwen2.5-32B). TTT no-QK consistently outperforms the baseline across all risk levels, with the largest gap at low $\delta$.}
\label{fig:risk_compute}
\vspace{-2pt}
\end{figure}

\paragraph{TTT is essential, not just the architecture.}
Table~\ref{tab:ablation_core} ablates the TTT-Probe by comparing against two controls: (1) standard supervised training on the same architecture, and (2) random initialization, with and without online updates.
The ``standard'' variant trains the same no-QK probe architecture via standard Adam optimization, using the same learning rate ($10^{-3}$) and number of epochs (20) as the corresponding TTT variant.
At inference, the standard-trained probe applies a single forward pass per step without online updates.
The ``no meta-training'' variants use randomly initialized weights (with or without online updates); since no training occurs, these results are epoch-independent.

\begin{table}[ht]
\centering
\small
\caption{Core mechanism ablation (supervised, $\delta{=}0.1$). Full TTT achieves the best savings: standard supervised training on the same architecture underperforms the static baseline, and removing either meta-training or online updates substantially degrades performance. Rows marked with $*$ use random initialization.}
\label{tab:ablation_core}
\vspace{2pt}
\begin{tabular}{lccccc}
\toprule
Configuration & Architecture & Training & Online update & Savings & Error \\
\midrule
Full TTT (no-QK) & Linear & Meta-learn & \checkmark & \textbf{.475} & .110 \\
Standard (no-QK) & Linear & Supervised & & .239 & .095 \\
\midrule
Full TTT (QK, $d_h{=}128$) & QK proj. & Meta-learn & \checkmark & .414 & .103 \\
No meta-training$^*$ & QK proj. & None & \checkmark & .254 & .099 \\
No meta + no update$^*$ & QK proj. & None & & .173 & .091 \\
\midrule
Static Probe & PCA+LogReg & Supervised & & .380 & .105 \\
\bottomrule
\end{tabular}
\vspace{-1pt}
\end{table}

\textbf{Standard training is insufficient; TTT meta-learning is the key contributor.}
The standard-trained no-QK probe achieves only 23.9\% savings, substantially below the static baseline (38.0\%).
Without PCA, training a linear model on the full 5{,}120-dimensional embedding space is prone to overfitting on 150K training samples.
The standard-trained probe does not benefit from online adaptation at inference.
Replacing standard training with TTT meta-learning improves savings from 23.9\% to 47.5\% (a 2.0$\times$ increase).
The meta-learning outer loop trains the probe not just to classify correctness, but to produce scores that improve through online updates at inference.
The QK block isolates the contribution of each TTT component: starting from the full TTT-Probe (41.4\%), removing meta-training drops savings to 25.4\%, and further removing online updates drops savings to 17.3\%.
Without meta-learned initialization, online updates alone recover only a fraction of the performance, confirming that both components are essential.

Figure~\ref{fig:risk_compute} visualizes the full risk--savings tradeoff in the in-distribution setting.
The gap between TTT-Probe and the static baseline is most pronounced in the moderate-risk regime ($\delta \in [0.05, 0.2]$), where the online adaptation provides the largest marginal benefit.
At higher $\delta$, all methods converge as most steps can be skipped regardless of probe quality.

Together, meta-training and online updates are complementary: meta-training provides the foundation for calibration quality, while online updates enable instance-level adaptation.

\paragraph{Architecture variants.}

Table~\ref{tab:ablation_arch} evaluates alternative probe designs.
QK architecture variants achieve competitive in-distribution savings (0.41--0.45). LayerNorm, residual, and shared QK variants reach 0.449--0.451, approaching the no-QK probe (0.475).

\begin{table}[ht]
\centering
\vspace{-2pt}
\caption{Architecture ablation (supervised, $d_h{=}128$, $\delta{=}0.1$).}
\label{tab:ablation_arch}
\vspace{1pt}
\begin{tabular}{lccc ccccc}
\toprule
& & & \multicolumn{5}{c}{OOD Savings} \\
\cmidrule(lr){4-8}
Variant & Sav. & Err. & MATH & GPQA & A'24 & A'25 & A'26 \\
\midrule
QK ($d_h{=}128$) & .414 & .103 & .670 & .665 & .295 & .258 & .134 \\
+ LayerNorm & .451 & .095 & .697 & \textbf{.726} & .265 & .256 & .144 \\
+ LN + Residual & .450 & .094 & .697 & \textbf{.726} & .265 & .256 & .144 \\
+ Shared QK & .449 & .094 & .698 & .724 & .264 & .256 & .144 \\
+ Learnable $\eta$ & .421 & .109 & .679 & .656 & \textbf{.364} & .260 & .104 \\
+ MLP (2-layer) & .441 & .093 & \textbf{.717} & .633 & .258 & .231 & \textbf{.325} \\
\midrule
no-QK (ep20) & \textbf{.475} & .110 & .637 & .715 & .293 & \textbf{.265} & .198 \\
\bottomrule
\end{tabular}
\end{table}

The OOD results reveal complementary strengths across architectures.
The MLP variant achieves the highest MATH-500 savings (0.717) and AIME'26 (0.325), while LayerNorm variants lead on GPQA (0.726).
The no-QK probe is strongest on AIME'25 (0.265), while the learnable-$\eta$ variant achieves the best AIME'24 savings (0.364).

The no-QK architecture remains the recommended default due to its simplicity (only $d_\phi + 1$ parameters), stability across epochs, and consistently strong performance.
QK variants with LayerNorm offer a viable alternative when MATH-style OOD generalization is prioritized.

\paragraph{Projection dimension.}

\begin{table}[ht]
\centering
\small
\vspace{-1pt}
\caption{Effect of QK projection dimension (supervised, $\delta{=}0.1$).}
\label{tab:dhidden}
\vspace{1pt}
\begin{tabular}{lccc}
\toprule
$d_h$ & Parameters & Savings & Error \\
\midrule
32 & 328K & .440 & .105 \\
64 & 656K & .439 & .106 \\
128 & 1.3M & .414 & .103 \\
256 & 2.6M & .408 & .104 \\
512 & 5.2M & .398 & .099 \\
\midrule
no-QK & 5.1K & \textbf{.475} & .110 \\
\bottomrule
\end{tabular}
\vspace{-1pt}
\end{table}

Table~\ref{tab:dhidden} evaluates the effect of QK projection dimension.
The smallest QK dimension ($d_h{=}32$) achieves the highest QK savings (0.440) with the fewest parameters (328K), while larger dimensions show diminishing returns.
The no-QK variant, with only 5.1K parameters, outperforms all QK dimensions in savings.
This suggests that for the early-stopping task, a high-capacity projection is unnecessary; the raw hidden states already contain sufficient signal for adaptive confidence estimation.

\paragraph{Inner learning rate sensitivity.}
The no-QK probe is remarkably robust to the inner learning rate.
Across a 100$\times$ range (lr $\in \{0.001, 0.005, 0.01, 0.05, 0.1\}$), supervised savings vary only between 0.461 and 0.463, a fluctuation of less than 0.5\%.
In consistent mode, the invariance is even more extreme: all five learning rates produce \emph{identical} savings of 0.418.
This robustness eliminates a key hyperparameter from practical deployment.

\subsection{Step-Level vs.\ Token-Level Savings}
\label{sec:step_vs_token}

\begin{table}[ht]
\centering
\small
\vspace{-2pt}
\caption{Step-level vs.\ token-level savings ($\delta{=}0.1$, supervised). The two metrics are highly consistent across models.}
\label{tab:token}
\vspace{2pt}
\begin{tabular}{l ccc}
\toprule
Configuration & Step Sav. & Token Sav. & $\Delta$ \\
\midrule
\multicolumn{4}{l}{\textit{Qwen2.5-32B}} \\
Static Probe & .377 & .377 & .000 \\
TTT no-QK & .475 & .471 & $-$.004 \\
TTT QK ($d_h{=}128$) & .414 & .414 & .000 \\
\midrule
\multicolumn{4}{l}{\textit{QwQ-32B}} \\
Static Probe & .293 & .296 & +.003 \\
TTT no-QK & .394 & .397 & +.003 \\
TTT QK ($d_h{=}128$) & .376 & .380 & +.003 \\
\midrule
\multicolumn{4}{l}{\textit{Llama-3.3-70B}} \\
Static Probe & .352 & .372 & +.020 \\
TTT no-QK & .424 & .438 & +.014 \\
TTT QK ($d_h{=}128$) & .378 & .399 & +.021 \\
\bottomrule
\end{tabular}
\end{table}

We verify that step-level savings translate to comparable token-level savings.
Table~\ref{tab:token} compares the two metrics for representative configurations at $\delta{=}0.1$.
On Qwen and QwQ, step-level and token-level savings differ by less than 0.5 percentage points, confirming that reasoning steps are roughly uniform in length.
On Llama, token-level savings are 1--2 percentage points higher than step-level savings, indicating that later reasoning steps tend to be longer; early stopping thus saves proportionally more tokens than steps.
Given the close agreement, we report step-level savings throughout this paper.

\subsection{Epoch Selection}
\label{sec:epoch}

Table~\ref{tab:epoch} shows in-distribution savings at selected epochs.
The no-QK probe is stable across epochs due to its small parameter count (5.1K), so we select epoch~20 where savings are near-peak.
The QK variant ($d_h{=}128$, 1.3M parameters) peaks at epoch~10 and degrades at later epochs. 
We therefore use epoch~10 for all QK variants.

\begin{table}[ht]
\centering
\small
\caption{Savings at selected epochs (supervised, $\delta{=}0.1$). The no-QK probe is stable; QK peaks early and overfits.}
\label{tab:epoch}
\begin{tabular}{l ccccc}
\toprule
Epoch & 10 & 20 & 30 & 40 & 50 \\
\midrule
no-QK & .443 & \textbf{.475} & .471 & .476 & .464 \\
QK ($d_h{=}128$) & \textbf{.414} & .387 & .381 & .357 & .324 \\
\bottomrule
\end{tabular}
\end{table}

\begin{figure}[t]
\centering
\includegraphics[width=0.45\textwidth]{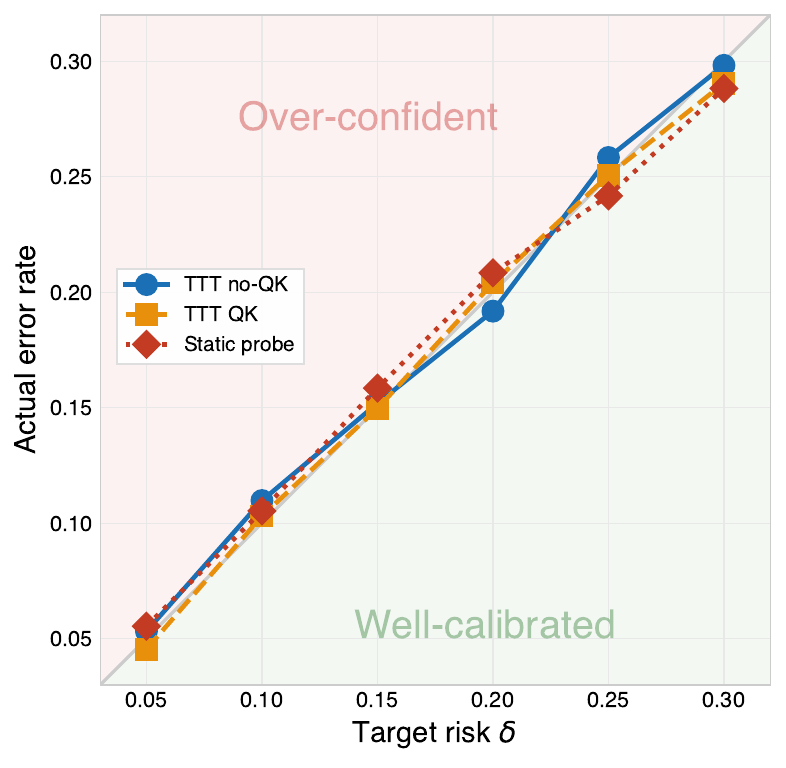}
\caption{Actual error rate vs.\ target risk $\delta$ (supervised, Qwen2.5-32B). All methods track the diagonal, confirming valid risk control. Points below the diagonal satisfy the LTT guarantee.}
\label{fig:calibration}
\end{figure}

\subsection{Calibration Quality}
\label{sec:calibration}

LTT calibration guarantees that the selected decision rule has deployment risk at most $\delta$ with probability at least $1-\epsilon$.
Figure~\ref{fig:calibration} validates this empirically by plotting the actual test-set error rate against the target $\delta$.

All three methods closely track the $y=x$ diagonal, confirming that the LTT guarantee holds empirically.
At low $\delta$ (0.05 to 0.15), all probes are slightly conservative, with actual error rates below the target risk level.
This is expected since LTT calibration optimizes for finite-sample validity.
Notably, TTT-Probe achieves higher savings (Figure~\ref{fig:risk_compute}) while maintaining the same calibration quality as the static baseline, indicating that online adaptation improves efficiency without degrading risk control.

Figure~\ref{fig:savings_dist} shows the distribution of per-problem savings at $\delta{=}0.1$.
All three methods exhibit high variance, with a significant mass near zero (problems where early stopping is not triggered) and near one (problems stopped very early).
The TTT no-QK distribution has a higher mean (0.475) and median (0.444) than the static baseline (mean 0.377, median 0.313), confirming that the improvement is broadly distributed across problems rather than driven by a few outliers.

\begin{figure}[t]
\centering
\includegraphics[width=0.5\textwidth]{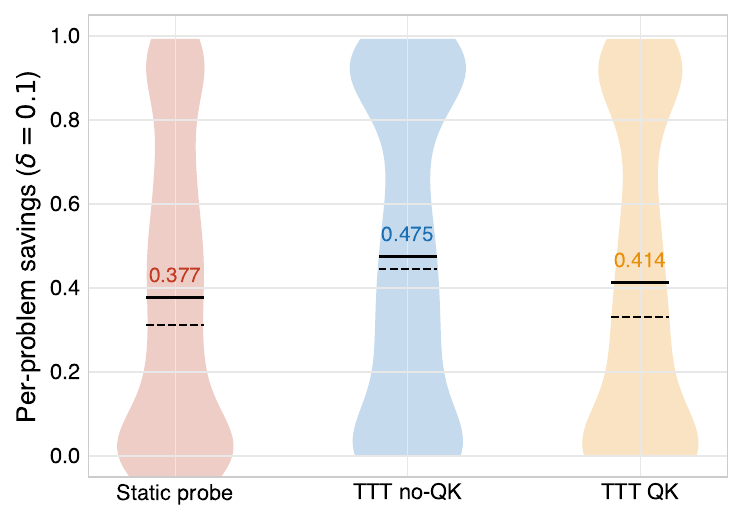}
\caption{Distribution of per-problem savings at $\delta{=}0.1$ (supervised, Qwen2.5-32B, 902 problems). Solid lines: mean; dashed lines: median. TTT no-QK shifts the distribution toward higher savings across the full range.}
\label{fig:savings_dist}
\end{figure}

\subsection{Score Trajectory Analysis}
\label{sec:trajectory}

Figure~\ref{fig:trajectory} compares the probe score trajectories on a representative test problem.
The green vertical line marks the first correct reasoning step.

The static probe score rises gradually after the first correct step but remains below its threshold (0.77) throughout the entire trajectory.
As a result, reasoning runs to completion with zero savings.
The TTT no-QK probe starts at a higher score ($\sim$0.65) due to meta-learned initialization and adapts online via the $C_t{=}0$ update rule.
After the first correct step, the score increases and crosses the calibrated threshold (0.83) at step~22, stopping reasoning 16 steps early and saving 41\% of compute.
This example illustrates the two components at work: meta-training provides a good initialization, and online updates adapt the probe to the specific problem instance.

\begin{figure}[t]
\centering
\includegraphics[width=0.7\textwidth]{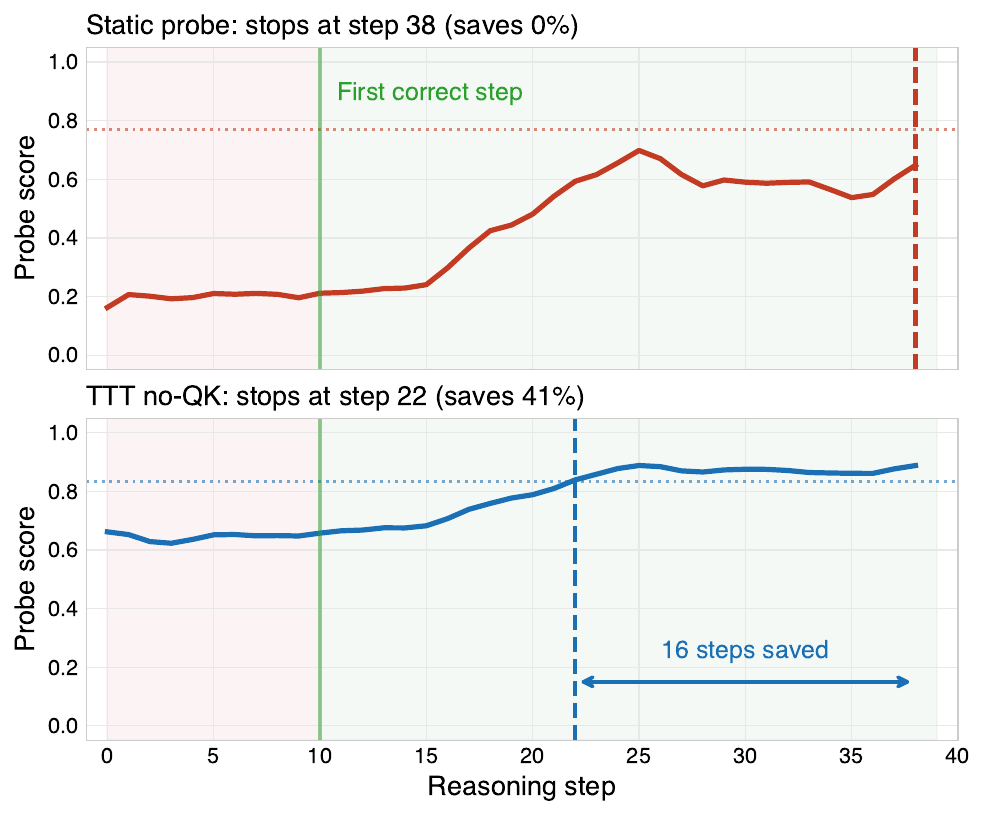}
\caption{Probe score trajectories for a test problem (Qwen2.5-32B, $\delta{=}0.1$). The green line marks the first correct step. The static probe (top) never crosses its threshold and saves 0\%. The TTT no-QK probe (bottom) crosses the threshold at step~22 and saves 41\%.}
\label{fig:trajectory}
\end{figure}

\subsection{Effect of Training and Calibration Set Sizes}
\label{sec:data_scaling}

We study how performance scales with the amount of available data, which is particularly relevant for deployment settings where labeled data is scarce.
Following the setup of \Cref{tab:main} (TTT no-QK, supervised, $\delta{=}0.1$), we vary the training set size $n_{\text{train}}$ with the calibration set fixed at $n_{\text{cal}}{=}1$K, and vary the calibration set size $n_{\text{cal}}$ with the training set fixed at $n_{\text{train}}{=}3$K.
Results are reported on the in-distribution test split.

\begin{table}[ht]
\centering
\small
\caption{Effect of training set size (top, $n_{\text{cal}}{=}1$K fixed) and calibration set size (bottom, $n_{\text{train}}{=}3$K fixed) on in-distribution performance (TTT no-QK, supervised, $\delta{=}0.1$). Savings grow monotonically with $n_{\text{train}}$; small calibration sets yield more conservative thresholds, and risk control is maintained at every size.}
\label{tab:data_scaling}
\vspace{2pt}
\begin{tabular}{l cccccc}
\toprule
$n_{\text{train}}$ & 500 & 1000 & 1500 & 2000 & 2500 & 3000 \\
\midrule
Savings & .385 & .424 & .438 & .450 & .458 & \textbf{.475} \\
Error   & .092 & .093 & .101 & .102 & .111 & .110 \\
\bottomrule
\end{tabular}
\\[8pt]
\begin{tabular}{l ccccc}
\toprule
$n_{\text{cal}}$ & 50 & 100 & 200 & 500 & 1000 \\
\midrule
Savings & .125 & .396 & .367 & .399 & \textbf{.475} \\
Error   & .024 & .091 & .082 & .091 & .110 \\
\bottomrule
\end{tabular}
\end{table}

\paragraph{Training set scaling.}
As shown in \Cref{tab:data_scaling} (top), savings increase monotonically with $n_{\text{train}}$ while the realized error stays close to the target level $\delta{=}0.1$ at every size.
Notably, with only $n_{\text{train}}{=}500$ the TTT-Probe already matches the Static Probe baseline trained on the full data (.385 vs.\ .380), and the full 3K training set improves savings to .475.
The method therefore does not rely on a large training set to be competitive; additional data converts directly into additional savings.

\paragraph{Calibration set scaling.}
The distribution-free FWER guarantee of LTT holds at any calibration set size.
As shown in \Cref{tab:data_scaling} (bottom), with $n_{\text{cal}}{=}50$ the binomial test naturally selects a more conservative threshold, yielding an error rate of .024 (well below the target $\delta{=}0.1$) at the cost of savings (.125).
From $n_{\text{cal}}{\geq}100$ onward, savings stabilize near .37--.40, reaching .475 at the full 1K.
The method does not fail in low-data deployment but becomes appropriately conservative.

\subsection{Stability Across Repeated Runs}
\label{sec:variance}

To verify that our results are not artifacts of a single run, we repeat the full pipeline with 10 random seeds under the setup of \Cref{tab:ood} ($\delta{=}0.1$): each seed trains its own probe and runs an independent LTT calibration.
\Cref{tab:variance} reports the mean and standard deviation on the in-distribution test split and the five OOD benchmarks.

\begin{table}[ht]
\centering
\small
\caption{Mean $\pm$ std over 10 random seeds with independent LTT calibration per seed (TTT no-QK, supervised, $\delta{=}0.1$). Variance is small on all benchmarks.}
\label{tab:variance}
\vspace{2pt}
\begin{tabular}{l cc}
\toprule
Dataset & Savings & Error \\
\midrule
In-distribution & .474 $\pm$ .004 & .112 $\pm$ .003 \\
MATH-500        & .631 $\pm$ .005 & .022 $\pm$ .001 \\
GPQA-Diamond    & .715 $\pm$ .002 & .300 $\pm$ .000 \\
AIME'24         & .292 $\pm$ .004 & .150 $\pm$ .000 \\
AIME'25         & .267 $\pm$ .005 & .056 $\pm$ .000 \\
AIME'26         & .199 $\pm$ .002 & .050 $\pm$ .000 \\
\bottomrule
\end{tabular}
\end{table}

Variance is small on all benchmarks: the standard deviation of savings is at most .005, the error rate is constant across seeds on GPQA-Diamond and all three AIME benchmarks, and the single-run numbers in \Cref{tab:ood} are consistent with these results.
Both the savings and the error rates of our method are thus stable under the randomness of probe training and threshold calibration.

\end{document}